\documentclass{article} % For LaTeX2e
\usepackage{iclr2025_delta,times}

\usepackage[ruled,vlined]{algorithm2e}
\usepackage{times}
\usepackage{soul}
\usepackage{url}
\usepackage[hidelinks]{hyperref}
\usepackage[utf8]{inputenc}
\usepackage[small]{caption}
\usepackage{graphicx}
\usepackage{amsmath}
\usepackage{amsthm}
\usepackage{booktabs}
\usepackage[switch]{lineno}
\usepackage{xcolor}
\usepackage{amsfonts}
\usepackage{wasysym}
\usepackage{subcaption}
\usepackage{hyperref}
\usepackage{url}
\usepackage{wrapfig}
\usepackage{setspace}
\usepackage{makecell}

\makeatletter
\renewcommand{\algocf@linesnumbered}{\small}
\makeatother

\newtheorem{proposition}{Proposition}

\def\defeq{\stackrel{\text{def}}{=}}

\DeclareMathOperator{\proba}{\mathbb{P}}

\newcommand{\jproba}[2]{\proba_{#1,#2}}
\newcommand{\mproba}[2]{\proba_{#1} \! \otimes \proba_{#2}}

\DeclareMathOperator{\expect}{\mathbb{E}} % Математическое ожидание.
\DeclareMathOperator{\MI}{\mathsf{I}}

\DeclareMathOperator*{\KLoperator}{\mathsf{KL}}
\newcommand{\KL}[2]{\KLoperator \left[ #1 \, \Vert \, #2 \right]}

\def\PMF{\pi}

\def\ourestname{DBMI}

\DeclareMathOperator*{\argmin}{arg\,min}
\usepackage{amsmath,amsfonts,bm}

\def\eqref#1{equation~\ref{#1}}
\def\1{\bm{1}}

\DeclareMathAlphabet{\mathsfit}{\encodingdefault}{\sfdefault}{m}{sl}
\SetMathAlphabet{\mathsfit}{bold}{\encodingdefault}{\sfdefault}{bx}{n}

\title{Discrete Bridges for Mutual Information \\ Estimation}

\author{%
  Iryna Zabarianska\thanks{Equal contribution. \\ \hspace*{0.1mm} $\quad \dagger$ Moscow Independent Research Institute of Artificial Intelligence \\ \hspace*{0.1mm} $\quad$ Corresponding author: \texttt{<kholkinsd@gmail.com>}} \\
  MIRAI$^\dagger$, Russia \\
  \And
  Sergei Kholkin$^*$  \\
  Applied AI Institute, Russia \\
  \And
  Grigoriy Ksenofontov \\
  Applied AI Institute, Russia, \\
  MIRAI$^\dagger$, Russia
  \And
  Ivan Butakov \\
  MIRAI$^\dagger$, Russia, \\
  Applied AI Institute, Russia, \\
  Institute of Numerical Mathematics (RAS), \hspace{-5mm} \\
  Russia \\
  \And
  Alexander Korotin \\
  Applied AI Institute, Russia, \\
  AXXX, Russia \\
}

\iclrfinalcopy % Uncomment for camera-ready version, but NOT for submission.
\begin{document}

\maketitle

\vspace{-8mm}
\begin{abstract}
Diffusion bridge models in both continuous and discrete state spaces have recently become powerful tools in the field of generative modeling. In this work, we leverage the discrete state space formulation of bridge matching models to address another important problem in machine learning and information theory: the estimation of the mutual information (MI) between discrete random variables.
By neatly framing MI estimation as a \textit{domain transfer} problem, we construct a Discrete Bridge Mutual Information (\textbf{DBMI}) estimator suitable for discrete data, which poses difficulties for conventional MI estimators. We showcase the performance of our estimator on two MI estimation settings: low-dimensional and image-based.
\end{abstract}

\vspace{-6mm}
\section{Introduction}
\vspace{-2mm}

\paragraph{Mutual Information (MI)} is a fundamental measure of nonlinear statistical dependence between two random vectors,
defined as the Kullback-Leibler divergence between the joint distribution $\jproba{X_0}{X_1}$ and the product of marginals $\mproba{X_0}{X_1}$ \cite{polyanskiy2024information_theory}:

% \vspace{-3mm}
\[
  \MI(X_0;X_1) = \KL{\jproba{X_0}{X_1}}{\mproba{X_0}{X_1}}
\]
% \vspace{-3mm}

Due to several outstanding properties, such as nullification only under statistical independence,
invariance to invertible transformations, and ability to capture non-linear dependencies,
MI is used extensively for theoretical analysis of overfitting \cite{asadi2018tightening_MI_gen_bounds},
hypothesis testing \cite{duong2022independence_testing},
feature selection
\cite{peng2005MI_feature_selection}
representation learning \cite{hjelm2018deep_infomax,butakov2025DIM_DM},
and studying mechanisms behind generalization in deep neural nets \cite{goldfeld2019estimating_information_flow_DNNs,butakov2024lossy_compression}. 
% shwartz_ziv2017opening_black_box,
% \vspace{-1mm}
While less explored compared to the continuous case, MI estimation between \textbf{discrete random vectors} is often leveraged in fields such as bioinformatics~\cite{newcomb2021coding_regions,xia2024peptide_sequencing}, neuroscience~\cite{chai2024brain}, and natural language processing~\cite{darrin2024cosmic}.
Despite conventional non-parametric methods offering strong baselines in discrete setups, they typically degrade as data dimensionality and alphabet size increase~\cite{pinchas2024discrete_comparison}.
This motivates the development of novel neural-based discrete MI estimators capable of capturing intricate interdependencies in complex high-dimensional settings.

% \sergei{Some words about the MI estimators? MI paragraph specially for finite stat space random variables?}

% \textcolor{red}{[Information Theory and Mutual Information, Ivan]}

\vspace{-2mm}
\paragraph{Bridge Matching.}

Diffusion models are a powerful type of generative models that show an impressive quality of data generation on continuous state spaces \cite{ho2020denoising}. However, they have some disadvantages, such as the inability to perform data-to-data translation via diffusion. To tackle this problem, a novel promising approach based on Reciprocal Processes  \cite{leonard2014reciprocal} and Schr\"{o}dinger Bridges theory \cite{de2021diffusion} has emerged. This approach is known as \textit{diffusion bridge matching} and is utilized for learning generative models as diffusion processes for \textit{domain translation}. This type of model has shown itself to be a powerful approach for many applications in biology \cite{bunne2023schrodinger}, chemistry \cite{somnath2023aligned}, computer vision \cite{liu2023_i2isb}, unpaired learning \cite{gushchin2024light}, and mutual information estimation \cite{kholkin2025infobridge}.

 While the approach was introduced for the continuous state space $\mathbb{R}^D$ it has recently been generalized to discrete state spaces, i.e., \textit{discrete bridge matching}, \cite{ksenofontovcategorical}. Such generalization proved itself useful in many applications: unpaired learning \cite{ksenofontovcategorical,kimdiscrete}, text generation \cite{gat2024discrete}, and drug design \cite{igashovretrobridge}.
 
% \textcolor{red}{[Add MI for finite state space variables intro and problem statement]}

\vspace{-2mm}
\paragraph{Contributions.} In this work, we tackle the problem of estimating the MI between discrete random variables of the same state space. We propose a novel MI estimator based on reciprocal processes and their representation as Markov chains (§\ref{sec:main_theory}). Leveraging this theoretical framework and the powerful generative methodology of Bridge Matching for discrete data domain transfer, we then develop a practical algorithm for MI estimation, named \ourestname{} (§\ref{sec:algorithm}). We evaluate our method on two benchmarks: a low-dimensional benchmark and a high-dimensional image-based benchmark (§\ref{sec:experiments}). The latter, introduced in this work, extends the evaluation to more complex discrete state spaces and is specifically designed to test the scalability of MI estimators while still providing access to ground-truth MI (\wasyparagraph\ref{sec:image_bench}). To our knowledge, it is one of the first to combine scalability with known ground truth in high-dimensional discrete settings. We further demonstrate the empirical superiority of our method over existing MI estimators suitable for discrete spaces.

\vspace{-2mm}
\paragraph{Notation.} We consider discrete state spaces of the form $\mathcal{X} = \mathbb{S}^D$, where $\mathbb{S} = {1, 2, \dots, S}$ denotes a categorical space and $D$ is the number of dimensions. This formulation is general and widely used in practice. 
% For simplicity, we refer to this space as $\mathcal{X}$ throughout the paper. 
Consider a \emph{time set} $\{t_n\}_{n=0}^{N+1}$, where $0=t_{0}<t_{1}<\dots<t_{N} < t_{N+1}=1$ are $N \geq 1$ time moments. The space $\mathcal{X}^{N+2}$ is referred to as the \emph{path space} and represents all possible trajectories $(x_0, x_{\text{in}}, x_{t_{N+1}})$, where $x_{\text{in}} \defeq (x_{t_1}, \dots, x_{t_N})$ corresponds to the intermediate states. Let $\mathcal{P}(\mathcal{X}^{N+2})$ be the space of probability distributions over paths. Each $r\in\mathcal{P}(\mathcal{X}^{N+2})$ can be interpreted as a discrete in time $\mathcal{X}$-valued stochastic process. We use $r(x_0, x_{\text{in}}, x_{t_{N+1}})$ to denote its probability mass function (PMF) at $(x_0, x_{\text{in}}, x_{t_{N+1}}) \in \mathcal{X}^{N+2}$ and use $r(\cdot|\cdot)$ to denote its conditional distributions, e.g., $r(x_1|x_0)$, $r(x_{\text{in}}|x_0,x_1)$. Finally, we introduce $\mathcal{M}(\mathcal{X}^{N+2}) \subset \mathcal{P}(\mathcal{X}^{N+2})$ as the set of all \emph{Markov processes} $r$, i.e., those processes which satisfy the equality $r(x_0, x_{\text{in}}, x_{t_{N+1}})=r(x_0)\prod_{n=1}^{N+1}r(x_{t_n}|x_{t_{n-1}})$.
To denote the probability mass function of $r\in\mathcal{P}(\mathcal{X}^{N+2})$ at a point $x_{t_n}\in\mathcal{X}$ at time $t_{n}$, we use $r(x_{t_n})$. 
We write $\KL{\cdot}{\cdot}$ to denote the Kullback-Leibler divergence between two distributions or stochastic processes. 

\vspace{-2mm}
\section{Background}
\vspace{-2mm}

In this section, we deliver background information on Mutual Information (\wasyparagraph\ref{sec:mi_background}) and explain concepts vital to our main result: Reciprocal processes (\wasyparagraph\ref{sec:reciprocal_processes}), their Markov chain representations (\wasyparagraph\ref{sec:reciprocal_cond} and \wasyparagraph\ref{sec:reciprocal_Markov_chain}) and the concept of Bridge Matching for discrete state spaces (\wasyparagraph\ref{sec:cat_BM}). 

% \alex{some explanaition what is going on in this section. Solved}

% \alex{D capital or d small everywhere. Synchronize}

\vspace{-2mm}
\subsection{Mutual Information}\label{sec:mi_background}
\vspace{-2mm}

%\alex{too much, too formal (especially for discrete variables). Do we consider the same state space? Simplify as much as possible. Not R but S, etd. DO we define this as mathhP or pi? may be synchronize with infobridge?}
% \textcolor{red}{[Probably a copy from InfoBridge]}

% Let $(\Omega, \setfamily, \proba)$ be a probability space with sample space $\Omega$, $\sigma$-algebra $\setfamily$,
% and probability measure $\proba$ defined on $\setfamily$.
% Consider random vectors $X \colon \Omega \to \reals^{D}$ and $Y \colon \Omega \to \reals^{D}$ with joint distribution $\jproba{X}{Y}$ and marginals $\proba_X$ and $\proba_Y$, respectively.
% Wherever it is needed, we assume the relevant Radon-Nikodym derivatives exist.
% For any probability measure $\qroba$ that is absolutely continuous w.r.t. $\proba$ (denoted $\qroba \ll \proba$),
% the Kullback-Leibler (KL) divergence is $\KL{\qroba}{\proba} = \expect_{\qroba} \left[\log \frac{\dif \qroba}{\dif \proba}\right]$,
% which is non-negative and vanishes if and only if (iff) $\proba = \qroba$.
% For discrete $X_0 \colon \Omega \to \mathcal{X}_0$ and $X_1 \colon \Omega \to \mathcal{X}_1$,
% the mutual information (MI) between $X_0,X_1$ quantifies the divergence between the joint distribution and the product of marginals:

For discrete $\mathcal{X}$-valued random variables $X_0, X_1$, with existent joint probability mass function $\pi(x_0, x_1)$,
the mutual information (MI) quantifies the KL divergence between the joint distribution and the product of marginals:
\vspace{-2mm}
\begin{gather}
  \MI(X_0;X_1) = \expect_{\pi(x_0, x_1)} \log \frac{\PMF(x_0,x_1)}{\PMF(x_0) \PMF(x_1)} = \KL{\PMF(x_0,x_1)}{\PMF(x_0) \PMF(x_1)} \label{eq:mi_def} 
\end{gather}

%If $\proba_X$ is discrete with PMF $\PMF(X)$, the Shannon's entropy of $X$ is defined as $\sent(X) = - \expect[\log \PMF(X)]$, where $\log(\blankarg)$ denotes the natural logarithm. Likewise, the joint entropy $\sent(X, Y)$ is defined via the joint PMF $\PMF(X, Y)$, and conditional entropy is $\sent(X \mid Y) = - \expect[\log \PMF(X \mid Y)] = - \expect_Y [\expect_{X \mid Y} \log \PMF(X \mid Y)]$. 
%In this case, MI satisfies
%\begin{equation}
%    \begin{aligned}
%        \MI(X;Y) &= \sent(X) - \sent(X \mid Y)
%        = \sent(Y) - \sent(Y \mid X) \\
%        &= \sent(X) + \sent(Y) - \sent(X,Y).
%    \end{aligned}
%\end{equation}
Mutual information is zero if and only if $X_0$ and $X_1$ are independent.
It is also invariant under injective transforms and admits many other noticeable properties:
data processing inequality, chain rule, etc.~\cite{polyanskiy2024information_theory}.

% \sergei{Add some MI properties here}

% \sergei{TODO: change \textbackslash PMF macro.}

\subsection{Reciprocal Processes}\label{sec:reciprocal_processes}
% \alex{rephrase with deepseek, too similar to infobridge}
% \textcolor{red}{[This section would be more complicated then in  InfoBridge]}
Reciprocal processes form a class of stochastic processes that have recently attracted increasing interest in various domains, including stochastic optimal control~\cite{leonard2014reciprocal}, Schrödinger bridge problems~\cite{de2021diffusion}, and diffusion-based generative modeling~\cite{gushchin2024light}. Although traditionally defined in continuous time, recent works have adopted discrete-time formulations~\cite{gushchin2024adversarial,ksenofontovcategorical}, favoring their simplicity and modeling flexibility. In this work, we consider a discrete-time reciprocal process, defined via a reference Markov process $q^{\rm ref} \in \mathcal{M}(\mathcal{X}^{N+2})$.

% In this work, we consider a discrete-time reciprocal process, defined via a reference Markov process $q^{\rm ref} \in \mathcal{M}(\mathcal{X}^{N+2})$.

Consider a distribution $\pi(x_0, x_1)\in \mathcal{P}(\mathcal{X}^2)$, reference process $q^{\rm ref}$ and define the discrete in time stochastic process $r_\pi$:
\begin{equation}
    r_\pi(x_0, x_{\rm in}, x_1) \defeq q^{\rm ref}(x_{\rm in}|x_0, x_1)\pi(x_0, x_1). \label{eq:reciprocal_def}
\end{equation}

% , which is consistent with Doob'h transform continuous in time methodology \cite{palmowski2002technique}.

The process $r_\pi$ can be interpreted as the reference process $q^{\rm ref}$ constrained to have marginal distribution $\pi(x_0, x_1)$ at its endpoints, while the inner part, i.e., conditioned on start $x_0$ and end $x_1$, reference process $q^{\rm ref}(x_{\rm in}|x_0, x_1)$, is also known as a \textit{bridge}. In continuous time, this construction aligns with the classical Doob $h$-transform  framework~\cite{palmowski2002technique}. Due to the non-causal nature of trajectory formation, process $r_\pi$ is, in general, not Markov. The set of all such $r_\pi$ for a particular $ q^{\rm ref}$ can be described as:
\begin{equation*}
    \mathcal{R}(q^{\rm ref})\!=\!\{r\in\mathcal{P}(\mathcal{X}^{N+2}) \; s.t.  \; \exists \pi \in \mathcal{P}(\mathcal{X}^2): r = r_\pi\}.
\end{equation*}
and we call it a set of \textbf{reciprocal processes} for $q^{\rm ref}$.

\vspace{-2mm}
\subsection{Reciprocal Processes Conditioned on a Point}\label{sec:reciprocal_cond}
\vspace{-2mm}

% \alex{check that all titles start with capital letters (all words)}

Consider a reciprocal process $r_\pi$ and its corresponding conditional process $r_\pi(x_{\rm in}, x_1|x_0) = r_{\pi|x_0}(x_{\rm in}, x_1)$.
% , which also remains reciprocal. 
One can be observed, process $r_{\pi|x_0}(x_{\rm in}, x_1)$ is Markov:
\begin{proposition}
    Consider the reciprocal process conditioned on point $x_0$,  $r_{\pi|x_0}(x_{\rm in}, x_1)$. Then $r_{\pi|x_0}(x_{\rm in}, x_1)$ is Markov:
    \vspace{-4mm}
    \begin{equation}
        r_{\pi|x_0}(x_{\rm in}, x_1)= \prod_{n=1}^{N+1}r_{\pi|x_0}(x_{t_{n}}|x_{t_{n-1}}) \label{eq:cond_reciprocal_Markovian}
    \end{equation}
\end{proposition}
\vspace{-2mm}
    
\noindent See Appendix~\ref{appx:proofs} for proof. Conditioning of reciprocal process on the start point can be generalized to reciprocal process starting from $\delta(x_0)$, i.e., Dirac delta distribution, and is a rare case of a stochastic process when both Markov and reciprocal properties are present. This also holds in continuous time, where the processes are known as Schrödinger–Föllmer processes~\cite{vargas2023bayesian_follmer}.

% This is also a known fact for the continuous in time formulation and such processes are known as Schr\"{o}dinger-F\"{o}llmer processes \cite{vargas2023bayesian_follmer}.
% \alex{Fix all half-lines}

\vspace{-2mm}
\subsection{Reciprocal Processes as Conditioned Markov Chains}\label{sec:reciprocal_Markov_chain}
\vspace{-2mm}

Let us note that one can naturally represent the whole $r_\pi$ as the mixture of $r_{\pi|x_0}$ over initial states $x_0$ and decompose  $r_{\pi|x_0}$ as the Markov chain, see \eqref{eq:cond_reciprocal_Markovian}:

\vspace{-3mm}
\begin{gather}
    r_\pi(x_0, x_{\rm in}, x_1) =  \pi(x_0) \prod_{n=1}^{N+1}r_{\pi|x_0}(x_{t_{n+1}}|x_{t_{n}}) = \pi(x_0) \prod_{n=1}^{N+1}r_{\pi} (x_{t_{n+1}}|x_{t_{n}}, x_0) \label{eq:full_reciprocal_Markovian}
\end{gather}
\vspace{-3mm}

% \alex{sum over what? here and everywhere else}
This represents $r_\pi$ as a \textit{Markov process family} conditioned on $x_0$, i.e., $r_{\pi|x_0}$. In that light, if one does know all the transition probabilities $r_\pi(x_{t_{n+1}}|x_{t_{n}}, x_0)$ and marginal $\pi(x_0)$, then one can sample from $r_{\pi}$ by performing the inference of Markov chain.

\vspace{-2mm}
\subsection{Bridge Matching for Discrete State Spaces}\label{sec:cat_BM}
\vspace{-2mm}

The $r_{\pi|x_0}$ admits a Markov chain formulation as given in~\eqref{eq:cond_reciprocal_Markovian} and it has analytic form:
\begin{equation}
    r_{\pi|x_0}(x_{t_n}| x_{t_{n-1}}, x_0)\!=\!\mathbb{E}_{\pi(x_1|x_0)}\!\left[ q^{\rm ref}(x_{t_{n}}|x_{t_{n-1}}, x_1) \right]\!, \label{eq:Markov_tranisiont_analytic}
\end{equation}
where $ q^{\rm ref}(x_{t_{n}}|x_{t_{n-1}}, x_1)$ is also known as \textit{posterior} and is usually analytically known and easily computable in exact form \cite{austin2021structured}. However, the expectation in \eqref{eq:Markov_tranisiont_analytic} is taken with respect to $\pi(x_1|x_0)$, which in practice is typically available only through empirical data samples; consequently, an exact evaluation of \eqref{eq:Markov_tranisiont_analytic} is rarely feasible.

Fortunately, they can be recovered as the solution to the following problem, \cite[Prop.~3.3]{ksenofontovcategorical}:
\vspace{-5mm}
\begin{multline}
    \label{eq:opt_reciprovcal_cond}
    r_{\pi|x_0}=\argmin_{s}  \mathbb{E}_{\pi(x_1|x_0)}\big[ \sum_{n=1}^{N}\mathbb{E}_{q^{\rm ref}(x_{t_{n-1}}|x_0, x_1)} \\
    \KL{q^{\rm ref}(x_{t_n}|x_{t_{n-1}}, x_1)}{s(x_{t_n}|x_{t_{n-1}})} - \mathbb{E}_{q^{\rm ref}(x_{t_N}|x_0, x_1)}[\log s(x_1|x_{t_N})] \big],
\end{multline}
where $r_{\pi|x_0} \in \mathcal{M}(\mathcal{X}^{N+1})$ and each markov transition $r_{\pi|x_0}(x_{t_n}|x_{t_{n-1}}) \in \mathcal{P}(\mathcal{X})$, see \eqref{eq:cond_reciprocal_Markovian}. 

In addition, one can create a family of independent subproblems like \eqref{eq:opt_reciprovcal_cond} indexed by \textit{condition} $x_0$ to recover $r_\pi$ as the family of $r_{\pi|x_0}$. Such an approach is known as \textit{conditional bridge matching} \cite{zhoudenoising,igashovretrobridge}:
\vspace{-3mm}
\begin{multline}
    \label{eq:full_loss}
    r_\pi=\argmin_{s} \mathbb{E}_{\pi(x_1, x_0)}\big[ \sum_{n=1}^{N}\mathbb{E}_{q^{\rm ref}(x_{t_{n-1}}|x_0, x_1)} \\ 
    \KL{q^{\rm ref}(x_{t_n}|x_{t_{n-1}}, x_1)}{s(x_{t_n}|x_{t_{n-1}}, x_0)} - \mathbb{E}_{q^{\rm ref}(x_{t_N}|x_0, x_1)}[\log s(x_1|x_{t_N}, x_0)] \big], 
\end{multline}
% \vspace{-2mm}
where $r_{\pi} \in \mathcal{P}(\mathcal{X}^{N+2})$ at the optimum is the mixture of $r_{\pi|x_0}$ processes \eqref{eq:opt_reciprovcal_cond}. The optimization objective \eqref{eq:full_loss} closely resembles the variational bound used in diffusion models \cite{ho2020denoising,austin2021structured} and is usually solved with standard deep learning techniques \cite{austin2021structured,ksenofontovcategorical}. Common practice is to parametrize a conditional Markov chain $r_\pi$ using a neural network. Specifically, each transition is modeled as $r_\theta(x_{t_{n}}|x_{t_{n-1}}, x_0)$, where $x_{t_{n-1}}$ and $x_0$ serve as inputs and the network outputs a probability distribution over $\mathcal{X}$. That is $r_\theta: \mathcal{X} \times \mathcal{X} \rightarrow \Delta^{|\mathcal{X}|}$, where $\Delta^{|\mathcal{X}|}$ denotes the probability simplex over $\mathcal{X}$.

\vspace{-2mm}
\section{Related Work}
\vspace{-2mm}

We briefly review existing approaches to mutual information estimation, covering non-parametric, variational, and diffusion-based methods, and highlight the limitations of current techniques when applied to discrete state spaces.

\vspace{-2mm}
\paragraph{Non-parametric Estimators.}
Typical non-parametric entropy and mutual information estimators rely on empirical PMF estimates, offering strong baselines for low-dimensional problems with small state spaces~\cite{pinchas2024discrete_comparison}.
However, as task complexity increases, these approaches struggle to yield accurate estimates because many states occur infrequently or remain unobserved.
For example, MI calculated through naïve PMF estimation is always upper-bounded by $\log(\text{sample size})$ \cite{cover2012elements},
a severe limitation in large state spaces.

\vspace{-2mm}
\paragraph{Variational Estimators.}
Neural parametric approaches address this challenge by capturing the intrinsic structure of the data,
filling the missing states implicitly.
Currently, the most prominent class of neural MI estimators capable of handling discrete cases leverages variational bounds.
These include MINE~\cite{belghazi2018mine}, NWJ \cite{nguyen2010estimating_nwj}, InfoNCE~\cite{oord2019infoNCE}, $f$-DIME~\cite{letizia2024mutual_fDIME}, and similar methods.
Such estimators, however, are prone to other theoretical limitations~\cite{mcallester2020limitations_MI} and currently lag behind more elaborate approaches based on diffusion models and normalizing flows~\cite{franzese2024minde,butakov2024normflows}.
On the other hand, these advanced methods target continuous distributions and do not apply to the discrete case --- a gap we fill in our work.

\vspace{-2mm}
\paragraph{Diffusion Based Estimators.} 
A recent class of MI estimators is based on diffusion processes, including MINDE~\cite{franzese2024minde}, which treats MI estimation as a \textit{generative modeling} problem, and InfoBridge~\cite{kholkin2025infobridge}, which approaches it as \textit{domain transfer}. These methods leverage tools from stochastic calculus, such as the disintegration and Girsanov theorems, to express $\KL{\pi(x_0, x_1)}{\pi(x_0)\pi(x_1)}$ as a KL divergence between diffusion processes.

% One of the most recent types of MI estimators is based on diffusion processes. These include MINDE \cite{franzese2024minde} that frames MI estimation as the \textit{generative modeling} problem and InfoBridge \cite{kholkin2025infobridge} that frames MI estimation as \textit{domain transfer} problem . These methods leverage stochastic calculus advances such as the Disintegration theorem and the Girsanov theorem to calculate the $\KL{\pi(x_0, x_1)}{\pi(x_0)\pi(x_1)}$ as the KL between diffusion processes.

The MI estimation is the result of MSE integration over the stochastic process trajectory. Such methods are known to be particularly suitable for complex data and estimation of high mutual information. While a discrete-state variant of MINDE, Info-SEDD~\cite{foresti2025info}, has recently been proposed, diffusion-based estimators have so far been developed primarily for continuous state spaces, limiting their applicability in discrete settings.

\vspace{-2mm}
\section{\ourestname{}. Mutual Information Estimator}
\vspace{-2mm}

In \wasyparagraph\ref{sec:main_theory}, we propose our novel DBMI method to estimate MI between discrete state space random variables, which is based on computing the KL divergence between reciprocal processes. Next, in \wasyparagraph\ref{sec:algorithm} we explain the practical implementation of our proposed MI estimator.

\subsection{Computing the MI though Reciprocal Processes}\label{sec:main_theory}

\vspace{-3mm}
Consider the MI estimation for discrete random variables $X_0, X_1 \in \mathcal{P}(\mathcal{X})$ with corresponding joint distribution $\pi(x_0, x_1)$. To tackle this problem the key idea  we choose the reference process $q^{\rm ref}$ and employ corresponding reciprocal processes: 

\vspace{-3mm}
\begin{gather}
    r_\pi^{\rm joint}(x_0, x_{\rm in}, x_1) \defeq q^{\rm ref}(x_{\rm in}|x_0, x_1)\pi(x_0, x_1), \label{eq:recirpocal_joint} \\
    r_\pi^{\rm ind}(x_0, x_{\rm in}, x_1) \defeq q^{\rm ref}(x_{\rm in}|x_0, x_1)\pi(x_0)\pi(x_1). \label{eq:recirpocal_ind} 
\end{gather}
\vspace{-3mm}

We show that the Mutual Information, as the KL between $\pi(x_0)\pi(x_1)$ and $\pi(x_0, x_1)$ \eqref{eq:mi_def}, is equal to the KL between reciprocal processes $r^{\rm joint}_\pi$ and $r_\pi^{\rm ind}$, i.e., discrete bridge matching models solving domain transfer problem. The latter can be decomposed into the KL between Markov chains.

% We show that the Mutual Information \eqref{eq:mi_def} can be represented as the KL divergence between reciprocal processes $r^{\rm joint}_\pi$ and $r_\pi^{\rm ind}$, i.e., discrete bridge matching models solving domain transfer problem. The latter can be decomposed into the KL divergence between Markov chains.

% \vspace{-10mm}
\begin{algorithm}[!t]
% \hspace{-20mm}
    \caption{\ourestname{}. Learning procedure.}\label{alg:train_model}
    \SetAlgoLined
    \KwIn{Distribution $\pi(x_0, x_1)$ accessible by samples, neural network parametrization $r_\theta$ of conditional Markov transitions $r_\pi^{\rm joint}$ and $r_\pi^{\rm ind}$, batch size $K$, number of inner samples $M$}
    \KwOut{Learned neural network $r_\theta$ approximating processes $r_\pi^{\rm joint}$ and $r_\pi^{\rm ind}$}
    
    \Repeat{converged}{
        Sample batch of pairs $\{x_0^k, x_{1,v=1}^k\}_{k=0}^K \sim \pi(x_0, x_1)$\;
        Make a shuffle $\{x_{1,v=0}^k\}_{k=0}^K = \text{Permute}(\{x_{1,v=1}^k\}_{k=0}^K)$; 
        \tcp{$v\!\!=\!\!0$ marks $r_\pi^{\rm ind}$, $v\!\!=\!\!1$ marks $r_\pi^{\rm joint}$} 
        Sample batch $\{n_k\}_{k=0}^K \sim U[1, N]$\; 
        Sample batch $\{x_{t_{n_k},v}^m\}_{m=0}^M \sim q^{\rm ref}(x_{t_{n_k}}|x_0^k, x_{1,v}^k)$ \;
        $\mathcal{L}_\theta = \frac{1}{KM} \sum\limits_{v\in\{0, 1\}} \sum\limits_{k=1}^{K}\sum\limits_{m=1}^{M}  \bigg[\mathbb{I}_{[n_k \neq N]}\cdot 
        \KL{q^{\rm ref}(x_{t_{n_k + 1}}|x^m_{t_{n_k},v}, x_{1,v}^k)}{r_{\theta}(x_{t_{n_k + 1}}|x^m_{t_{n_k},v}, x_0^k, v)} 
        -\mathbb{I}_{[n_k = N]}\cdot\log r_{\theta}(x_1|x_{t_{N},v}^m, x_0^k, v) \bigg]$ \;
        Update $\theta$ using $\frac{\partial \mathcal{L}_\theta}{\partial \theta}$
        }
\end{algorithm}

\begin{algorithm}[t!]
    \caption{\ourestname{}. Estimation procedure.}\label{alg:MI_estimator}
    \SetAlgoLined
    \KwIn{Distribution $\pi(x_0, x_1)$ accessible by samples, neural network parametrization $r_\theta$ of conditional Markov transitions $r_\pi^{\rm joint}$ and $r_\pi^{\rm ind}$, number of samples $K$,  number of inner samples $M$}
    \KwOut{Mutual information estimation $\widehat{\text{MI}}$}
    
    Sample batch of pairs $\{x_0^k, x_1^k\}_{k=1}^K \sim \pi(x_0, x_1)$; \\
    
    Sample batch $\{n_k\}_{k=0}^K \sim U[1, N]$;\\
    
    Sample $\{x_{t_{n_k}}^m\}_{m=0}^M\sim q^{\rm ref}(x_{t_{n_k}}|x_0^k, x_1^k)$; \\
    
    $\widehat{\text{MI}} \leftarrow \frac{1}{KM}\sum\limits_{k=1}^{K}\sum\limits_{m=1}^{M}\bigg[ \KL{r_{\theta}(x_{t_{n_k+1}}|x_{t_{n_k}}^m, x_0^k, 1)}{r_{\theta}(x_{t_{n_k+1}}|x_{t_{n_k}}^m, x_0^k, 0)}\bigg]$;\\
    
    \Return $\widehat{\text{MI}}$
\end{algorithm}

% \alex{proposition, not theorem}

\begin{proposition}[Mutual Information Decomposition]\label{th:mi_main}
    Consider random variables $X_0, X_1 \in \mathcal{P}(\mathcal{X})$ and their joint distribution $\pi(x_0, x_1)$. Consider reciprocal processes $r_\pi^{\rm joint}$ and $r_\pi^{\rm ind}$ induced by distributions $\pi(x_0, x_1)$ and $\pi(x_1)\pi(x_0)$ respectively, as in \eqref{eq:recirpocal_joint} and \eqref{eq:recirpocal_ind}. The the MI between the random variables $X_0$ and $X_1$ can be expressed as:
    \vspace{-3mm}
    \begin{gather}
        I(X_0;X_1) = \sum_{n=1}^{N}\mathbb{E}_{r_\pi(x_0, x_{t_n})} \Big[ \KL{r_\pi^{\rm joint}(x_{t_{n+1}}|x_{t_n}, x_0)}{r_\pi^{\rm ind} (x_{t_{n+1}}|x_{t_n}, x_0)}\Big], \label{eq:mutinfo_th}
    \end{gather}
    where
    \vspace{-6mm}
    \begin{gather}
        r_\pi^{\rm joint}(x_0, x_{\rm in}, x_1) = \pi(x_0) \prod_{n=1}^{N}r_\pi^{\rm joint}(x_{t_{n+1}}|x_{t_{n}}, x_0), \\[-1mm]
        r_\pi^{\rm ind}(x_0, x_{\rm in}, x_1) = \pi(x_0) \prod_{n=1}^{N}r_\pi^{\rm ind}(x_{t_{n+1}}|x_{t_{n}}, x_0). 
    \end{gather}
    % \vspace{-4mm}
    % \alex{DENSITY? Sum over what?}
The $r_\pi^{\rm joint}$ and $r_\pi^{\rm ind}$ are the conditional Markov transitions from the corresponding representations \eqref{eq:full_reciprocal_Markovian} of the reciprocal processes $r_\pi^{\rm joint}$ and  $r_\pi^{\rm ind}$.
\end{proposition}
\vspace{-2mm}

% \vspace{-2mm}
See Appendix~\ref{appx:proof_main_theorem} for proof. Once the Markov chain transitions $r_\pi^{\rm joint}$ and $r_\pi^{\rm ind}$ are known, our Proposition~\ref{th:mi_main} provides a straightforward way to estimate the mutual information between random variables $X_0$ and $X_1$ by evaluating the KL divergence between $r_\pi^{\rm joint}$ and $r_\pi^{\rm ind}$ at trajectory points $x_{t_n}, x_0 \sim r_{\pi}^{\rm joint}(x_0, x_{t_n})$. 

\vspace{-2mm}
\subsection{Practical Algorithm}\label{sec:algorithm}
\vspace{-2mm}

% \alex{somehow split training and inference(estimation)}

% \alex{add some details to apprndix, e.g., the approximation by posterior sampling, Lsimple}
% \sergei{Add posterior to appendix B.1. mention in main text.}

We developed a practical algorithm to estimate the mutual information based on Proposition~\ref{th:mi_main} and call it \textbf{\ourestname{}}. Our approach consists of two distinct stages: first, we learn the Markov transitions via bridge matching, (see \wasyparagraph\ref{sec:cat_BM} and Algorithm~\ref{alg:train_model}); second, we use these learned transitions to compute the mutual information (see Algorithm~\ref{alg:MI_estimator}).

The reciprocal processes $r_\pi^{\rm joint}$ and $r_\pi^{\rm ind}$ can be recovered by the Bridge Matching for discrete state spaces procedure, see \wasyparagraph\ref{sec:cat_BM}. We have to solve optimization problem~\eqref{eq:full_loss} by parametrizing $r_\pi^{\rm joint}$ and $r_\pi^{\rm ind}$ with neural networks $r_\pi^{\rm joint, \phi}$ and $r_\pi^{\rm ind, \psi}$, respectively, and applying Gradient Descent on Monte Carlo approximation of~\eqref{eq:full_loss}.

The MI estimate in~\eqref{eq:mutinfo_th} requires sampling from the reciprocal process $r_\pi^{\rm joint}(x_0, x_{t_n})$ at times $0$ and $t_n$, which is straightforward in a \textit{simulation-free} manner since:
% \vspace{-2mm}
\begin{gather}
    r_\pi^{\rm joint}(x_0, x_{t_n}) = \mathbb{E}_{\pi(x_1)}[r_\pi^{\rm joint}( x_0, x_{t_n}|x_1)] = \mathbb{E}_{\pi(x_1)}[q^{\rm ref}(x_{t_n}|x_0, x_1)\pi(x_0|x_1)].
\end{gather}
\vspace{-3mm}

Therefore, a sample from it can be obtained by sampling from \( \pi(x_0, x_1) \) and  \( q^{\rm ref}(x_{t_n}|x_0, x_1) \).

% We call our practical MI estimation algorithm \ourestname{}. The reciprocal process $r_\theta$ learning procedure is described in Algorithm~\ref{alg:train_model} and the MI estimation procedure is described in Algorithm~\ref{alg:MI_estimator}.

\vspace{-2mm}
\paragraph{Amortization of Learning Procedure.}

In practice, to ammortize the optimization procedure, we replace two separate neural networks that approximate the conditional
Markov transitions $r_\pi^{\rm joint}(x_{t_{n+1}}|x_{t_n}, x_0)$ and $r_\pi^{\rm ind}(x_{t_{n+1}}|x_{t_n}, x_0)$ with a single neural network that incorporates an additional binary input. Specifically, we introduce a binary input $v \in \{0, 1\}$ to unify the Markov chain transitions approximations in the following way:  
$r_\theta(\cdot, 1) \approx r_\pi^{\rm joint}(\cdot)$  
and  
$r_\theta(\cdot, 0) \approx r_\pi^{\rm ind}(\cdot)$.  
The introduction of such additional input is a common technique (see, for e.g., \cite{franzese2024minde}).

\vspace{-2mm}
\paragraph{Markov Transitions Parametrization.}

% \alex{repharase GPT}
Note that the size of the state space grows exponentially as $|\mathcal{X}| = |\mathbb{S}|^D = S^D$, where $S$ is the number of categories per variable and $D$ is the total number of variables. Consequently, a neural network used to parametrize the transition kernel $r_\theta(x_{t_n}|x_{t_{n-1}})$ would, in principle, need to output $S^D$ probabilities, which quickly becomes computationally prohibitive. Because directly modeling such transition distributions is impractical, we adopt a common strategy from discrete generative modeling~\cite{austin2021structured,sahoo2024simple} and impose a factorized structure over dimensions on the probabilities:
\vspace{-2mm}
\begin{equation}
    r_\theta(x_{t_n}|x_{t_{n-1}}, x_0) \approx \prod_{d=1}^D r_\theta(x_{t_n}^d|x_{t_{n-1}}, x_0).
\end{equation}
% \vspace{-2mm}
In this case, for each $x_{t_{n-1}}$, we predict a row-stochastic $D \times S$ matrix of probabilities $r_\theta(x_{t_n}^d \mid x_{t_{n-1}}, x_0)$. Rather than using separate models for each of the $N+1$ transitions, we train a single neural network with an additional timestep input. The network takes inputs of the form $S^D \times S^D \times [0,1]$ and outputs $\mathbb{R}^{S \times D}$. We also apply posterior sampling (see Appendix~\ref{appx:method_details} for details).

% In this case, for each $x_{t_{n-1}}$, we just need to predict a row-stochastic matrix $D \times S$ matrix of probabilities $r_\theta(x_{t_n}^d|x_{t_{n-1}}, x_0)$. In addition, we have to learn $N+1$ distinct probability transitions $r_\theta(x_{t_n}^d|x_{t_{n-1}}, x_0)$ and instead of using the distinct neural networks, we learn a single neural network with an additional input indicating the timestep. Thus, the neural network takes input in the form $S^D \times S^D \times [0, 1]$ and outputs $\mathbb{R}^{S \times D}$. In addition, we employ posterior sampling (see Appendix~\ref{appx:method_details} for additional details).

\vspace{-2mm}
\paragraph{Reference Process.} Reciprocal processes that we work with depend on the reference process $q^{\rm ref}$, see \wasyparagraph\ref{sec:reciprocal_processes}. We consider the case where we do not have any prior information on data so we choose the simplest reference process with \textit{uniform} categorical noise, i.e., random change of category, \cite{austin2021structured}. We start from one-dimensional case. Define a process where the state stays in the current category $x_{t_{n-1}}$ with high probability, while the remaining probability is distributed uniformly among all other categories. We call this process the \textit{uniform}. Formally, as a homogeneous Markov process, it is defined by the following probabilities:

\vspace{-4mm}
\begin{gather}
    q^{\rm unif}(x_{t_{n}}|x_{t_{n-1}})= \begin{cases}
        1 - \alpha, & x_{t_n} = x_{t_{n-1}}, \\
        \dfrac{\alpha}{S-1}, & x_{t_n} \neq x_{t_{n-1}},\label{eq:q_ref}
    \end{cases}
\end{gather}
% \vspace{-2mm}
where $\alpha \in [0, 1]$ is the \textit{stochasticity parameter} that controls the probability of transitioning into another state. Then, to construct such a process for $D > 1$  one has to combine several such independent one-dimensional processes. 

\vspace{-1mm}
For such case bridge $q^{\rm ref}(x_{t_n} |x_0, x_1)$ and posterior $q^{\rm ref}(x_{t_n}|x_{t_{n-1}}, x_1)$ is known analytically \cite{austin2021structured} and can be sampled very easily for all $t_n$. While alternative reference processes, such as those incorporating intercategory  dependencies~\cite{austin2021structured}, are possible, we adopt $q^{\rm unif}$ due to its simplicity and robustness, as it does not rely on any additional assumptions. 

\begin{table*}[!t]
    \vspace{-5mm}
    \centering
    \begin{tabular}{lccccccc}
        \toprule
        Dim & MI$_{true}$ & \ourestname{} (\textbf{ours}) & $f$-DIME-KL & $f$-DIME-H & $f$-DIME-G & MINE & InfoNCE \\
        \midrule
        $D=2$ & 3.38 &  \textbf{3.40} & 2.97 & 3.10 & 3.17 & 3.23 &  3.31  \\
        $D=5$ & 8.50  &   \textbf{8.67}      & 3.43 & 6.89 & 7.08 & 6.69 &   5.99   \\
        $D=10$ & 16.80 & \textbf{16.90}       & 5.04 & 7.01 & 9.23 & 3.28 &  6.22  \\
        $D=25$ & 42.36  &  \textbf{42.10}       & 4.35 & 4.69 & 9.33 & 9.83 &  6.22  \\
        $D=50$ & 85.47 & \textbf{79.30}        & 3.82 & 1.85 & 9.52 & 8.75 &  6.19   \\
        \bottomrule
    \end{tabular}
    \caption{Results on the low-dimensional benchmark for varying dimension $D$ (see Dim column) and fixed number of categories $S=10$. MI$_{\text{true}}$ denotes ground-truth MI; other columns show estimates from each method. The closer to MI$_{\text{true}}$ the better. The closest to GT estimates are marked in \textbf{bold}.}
    \label{tab:benchmark_cat10}
\end{table*}

\begin{table*}[!t]
    % \tiny
    \centering
    \begin{tabular}{lccccccc}
        \toprule
        NumCat  & MI$_{true}$ & \ourestname{} (\textbf{ours}) & $f$-DIME-KL & $f$-DIME-H & $f$-DIME-G & MINE & InfoNCE \\
        \midrule
        $S=2$ & 3.64 &    3.91     &  3.08  & 3.37 & \textbf{3.84} & 3.28 &   2.909   \\
        $S=5$ & 10.35 &    \textbf{10.21}     &   4.65  & 7.44 & 6.97 & 6.82 &  6.13  \\
        $S=10$ & 16.80 &   \textbf{16.90}  & 5.04 & 7.01 & 9.23 & 3.28 &   6.21  \\
        $S=25$ & 25.86 &  \textbf{25.33}        &  4.42  & -9.65 & 9.97 & 9.8 &   6.21  \\
        $S=50$ & 32.73 &  \textbf{31.03}    &  3.57  & 7.11 & 8.65 & 9.71 &   6.18  \\
        \bottomrule
    \end{tabular}
    \caption{Results on the low-dimensional benchmark for varying number of categories $S$ (see Dim column) and fixed dimension $D=10$. MI$_{\text{true}}$ denotes ground-truth MI; other columns show estimates from each method.
    The closer to MI$_{\text{true}}$ the better. The closest to GT estimates are marked in \textbf{bold}.}
    \label{tab:benchmark_dim10}
    \vspace{-2mm}
\end{table*}

\vspace{-3mm}
\section{Experiments}\label{sec:experiments}
\vspace{-2mm}

We evaluate our method on two benchmarks, where known ground truth MI values allow us to assess the performance of our method.
To cover low-dimensional cases, we construct joint probability distributions with a manageable number of categories and then stack these distributions upon each other (see \wasyparagraph\ref{sec:low_dim_bench}). Next, to evaluate the capability of our method to estimate mutual information between complex, high-dimensional random variables, we propose a procedure for embedding discrete state space variables into structured image data and estimating MI between them (see \wasyparagraph\ref{sec:image_bench}). We compare ours \ourestname{} to other popular neural MI estimators that are suitable for working with discrete random variables: MINE \cite{belghazi2018mine}, InfoNCE \cite{oord2019infoNCE}, NWJ \cite{nguyen2010estimating_nwj} and $f$-DIME \cite{letizia2024mutual_fDIME} with KL, Hellinger (H), and GAN (G) divergences, using the same deranged architecture across all divergences.

\vspace{-2mm}
\subsection{Low Dimensional Benchmark}\label{sec:low_dim_bench}
\vspace{-2mm}

% \textcolor{red}{[Irina]: \ourestname{}}, \textcolor{red}{[Ivan]: info about another estimators.}

We validation our method on distributions with known mutual information, constructed in the following way. Consider a random vector \((X_0,X_1) = (X^{1}_0, X^{1}_1, \ldots, X^{D}_0, X^{D}_1)\) taking values in \(\mathbb{S}^{2D}\). We factorize this random vector pdf $\pi(x_0, x_1)$ across dimensions:
\vspace{-2mm}
\begin{equation*}
    \pi(x_0 ,x_1) = \prod_{d=1}^{D} \pi_d(x^{d}_0, x^{d}_1),
\end{equation*}
where each bivariate factor \(\pi_d(x^{d}_0, x^{d}_1)\) is constructed by first drawing \(x^{d}\) uniformly over \(\mathbb{S}\) and then defining the conditional distribution via a randomly generated stochastic transition matrix \(\Pi(x^{d}_1|x^{d}_0)\) with controlled level of dependence between $X^d_0$ and $X^d_1$. 
Owing to the factorization across dimensions, the total mutual information decomposes as $I(X_0;X_1) = \sum_{d=1}^{D} I(X^{d}_0; X_1^{d})$.

\begin{figure*}[t]

    \vspace{-6mm}
    \centering

    \begin{subfigure}[t]{0.24\linewidth}
        \centering
        \includegraphics[width=\linewidth]{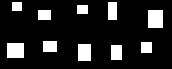}
        \caption{$x_0\sim\pi(x_0)$}
    \end{subfigure}
    \begin{subfigure}[t]{0.24\linewidth}
        \centering
        \includegraphics[width=\linewidth]{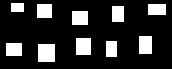}
        \caption{$x_1\sim\pi(x_1)$}
    \end{subfigure}
    \begin{subfigure}[t]{0.24\linewidth}
        \centering
        \includegraphics[width=\linewidth]{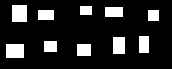}
        \caption{$x_1\sim r_{\theta}(\cdot, 0)$}
    \end{subfigure}
    \begin{subfigure}[t]{0.24\linewidth}
        \centering
        \includegraphics[width=\linewidth]{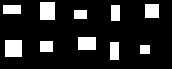}
        \caption{$x_1\sim r_{\theta}(\cdot, 1)$}
    \end{subfigure}

    \caption{Qualitative samples, presented in $5\times2$ image grids, generated using image benchmark and \ourestname{}, i.e.,  $r_{\theta}(\cdot)$, at resolution $32\times32$. }
    \label{fig:rectangles32}
    % \vspace{-4mm}
\end{figure*}
\begin{table*}[t]
    \vspace{-2mm}
    \centering
    \begin{tabular}{lcccccc}
        \toprule
        MI$_{true}$  & \ourestname{} (\textbf{ours}) & $f$-DIME-KL & $f$-DIME-H & $f$-DIME-G & MINE & NWJ \\
        \midrule
       $1$ & 0.504 & 0.668 & 0.687 & \textbf{0.701} & 0.0 & 0.0 \\
        $2$ & \textbf{1.426} & 1.158 & 1.264 & 1.330 & 0.0 & 0.0 \\
        $4$ & \textbf{3.923} & 2.211 & 2.676 & 2.719 & 1.84 & 0.685 \\
        $6$ & \textbf{6.301} & 2.957 & 3.678 & 3.385 & 2.784 & 1.789 \\
        $8$ & \textbf{8.447} & 4.239 & 4.919 & 4.886 & 3.776 & 2.554 \\
        \bottomrule
    \end{tabular}
    \caption{Results for the high-dimensional image benchmark  with $32\times32$ images across all estimators. MI$_{\text{true}}$ denotes the ground-truth mutual information. All the other columns correspond to MI estimation results for a particular method. The closer to MI$_{\text{true}}$ the better. The closest to GT estimates are highlighted in \textbf{bold}.}
    \label{tab:exp_rectangle}
    \vspace{-5mm}
\end{table*}
%The method is validated on synthetic distributions with known mutual information. We generate pairs of categorical random vectors \( X = (x^{1}, x^{2}, \ldots, x^{D}) \) and \( Y = (y^{1}, y^{2}, \ldots, y^{D}) \), where each dimension \( d \) is sampled independently. For each dimension, we sample \( x^{d} \) uniformly over $\mathbb{S}$ and construct a row-stochastic transition matrix \( P(y^{d}|x^{d}) \) with tunable diagonal strength via a stochasticity parameter \( \sigma \): smaller \( \sigma \) yields sharper diagonals (higher MI), larger \( \sigma \) yields flatter matrices (lower MI). The total mutual information then it is calculated as follows: \( I(X;Y) = \sum_{d=1}^D I(x^{d}; y^{d}) \). Other details are described in Appendix~\ref{appx:synthetic}.

We compare our \ourestname{} and  other mentioned MI estimators on low-dimensional benchmark, by increasing both $S=|\mathbb{S}|$ and \( D \) as shown in Table~\ref{tab:benchmark_cat10} and Table~\ref{tab:benchmark_dim10}. All the methods are trained using \( 10^4 \) train samples and \( 10^4 \) test samples. Our method employs a Diffusion Transformer (DiT)-based architecture trained for 150 epochs, $\alpha=10^{-4}$ for $q^{\rm ref}$ \eqref{eq:q_ref}, $M=1$ for training, Algoritm~\ref{alg:train_model}, and $M=10$ for MI estimation, Algoritm~\ref{alg:train_model}. For further details, see Appendix~\ref{appx:synthetic}.

%\textcolor{red}{That is not true!}
%\textit{We train all methods using \( 10^5 \) samples and a batch size of \( 512 \) for \( 150 \) epoch.}

The results demonstrate that \ourestname{} consistently outperforms competing methods. \ourestname{} holds precise MI estimation with the growth of dimension, see Table~\ref{tab:benchmark_cat10},  and the growth of number of categories, see Table~\ref{tab:benchmark_dim10}, while other neural estimators start to  fail very early.
% \alex{unfinished tables}
% \alex{why only neural?}

\vspace{-2mm}
\subsection{Image Benchmark}\label{sec:image_bench}
\vspace{-2mm}

To evaluate scalability, we propose the first image-based benchmark for discrete data. Building on the idea that MI is invariant w.r.t. bijective mappings \cite{butakov2024lossy_compression}. We construct low-dimensional latent random variables and corresponding bijective mapping to the \textit{image} space. This allows for precise MI control within a high-dimensional, complex image state space.

Specifically, we have four \emph{latent variables} $(X^1_0, X^2_0, X^3_0, X^4_0)$ that specify the pixel border coordinates of the single rectangle. We sample each coordinate as $X_0^i \sim U(0,\dots,V)$ and enforce a minimum side length $V^{\text{min}}$ to avoid degenerate rectangles. The target $X_1$ is obtained by passing each latent coordinate $X_0^i$ through a symmetric noisy channel \cite[\wasyparagraph4.5]{leebenchmark} independently, which yields following mutual information between images $\MI(X_0; X_1) = \sum_{i=1}^{4} \MI(X^i_0; X_1^i)$.

\begin{wrapfigure}{r}{0.6\textwidth}
\vspace{-5mm}
% \vspace{}
    \centering
    \includegraphics[width=0.95\linewidth]{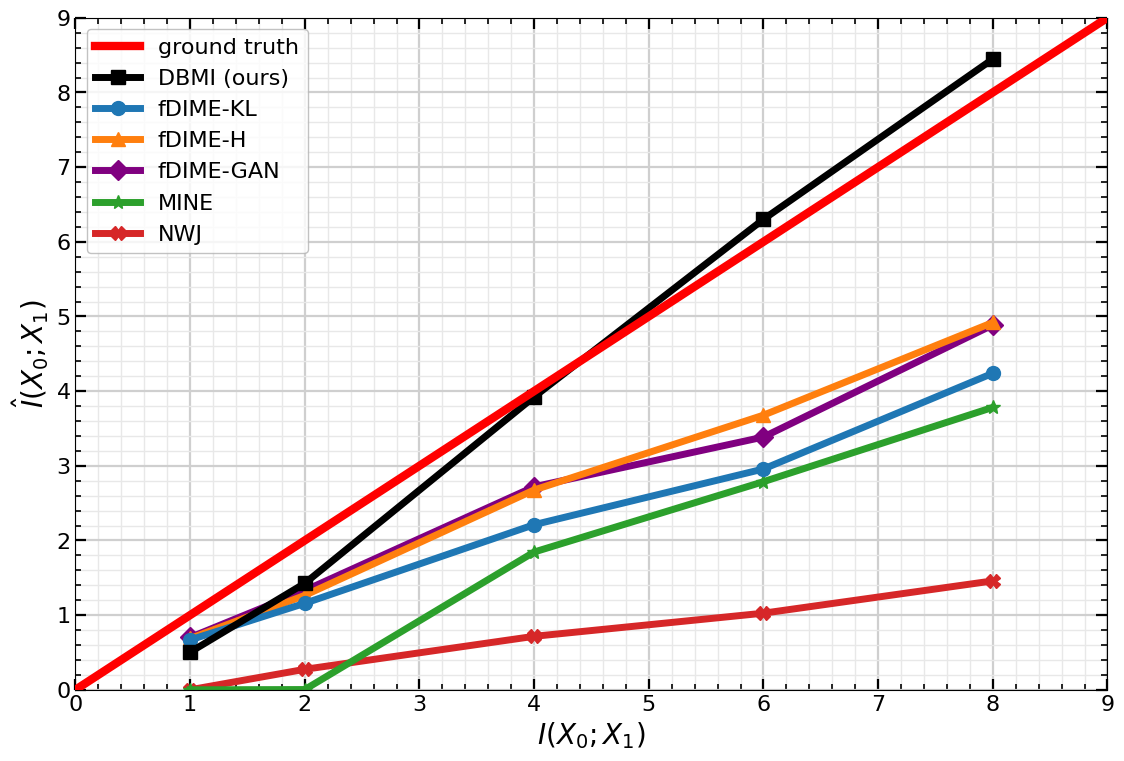}
    \caption{Comparison of estimated mutual information $\hat{\MI}(X_0; X_1)$ across methods against the ground-truth $\MI(X_0; X_1)$ (\textcolor{red}{red}) on a high-dimensional image benchmark with size $32\times32$.}
    % Shaded regions denote 99\% confidence intervals computed across multiple random-seed runs.}
    \label{fig:exp_rectangle}
    \vspace{-5mm}
\end{wrapfigure}

In our experiments, we set $V=10$, $V^{\text{min}} = 10$ and render $32 \times 32$ binary images ($S = 2$), where $1$ indicates the rectangle and $0$ the background. We train all methods on $10^5$ samples for mutual information values in $\{1,2,4,6,8\}$ and evaluate on $10^4$ validation samples. Specifically, we train our method for 20 epochs and use $\alpha=10^{-2}$ for $q^{\rm ref}$ in \eqref{eq:q_ref}, $M=1$ for training, Algoritm~\ref{alg:train_model}, and $M=10$ for MI estimation, Algoritm~\ref{alg:train_model}. Additional experiments with $16 \times 16$ images are presented in Appendix~\ref{appx:add_exp}.

Results in Figure~\ref{fig:exp_rectangle} and Table~\ref{tab:exp_rectangle} show that \ourestname{} scales to high-dimensional images and accurately estimates mutual information in regimes where competing methods fail. In addition to accurate MI estimation, our method learns to generate perfect samples from marginal $x_1$, see Figure~\ref{fig:rectangles32}.

\vspace{-2mm}
\section{Discussion}
\vspace{-2mm}

% \alex{Appendices: technical details, 1 proof, images from GK}
\paragraph{Potential Impact.} Our contributions include the development of a novel Mutual Information estimator for the discrete state space random variables grounded in the reciprocal processes and the bridge matching theory. The proposed algorithm, \ourestname{}, achieves superior performance over widely used mutual information estimators applicable for discrete state spaces, including MINE \cite{belghazi2018mine}, InfoNCE \cite{oord2019infoNCE}, and $f$-DIME \cite{letizia2024mutual_fDIME}, across both low-dimensional benchmark and challenging high-dimensional image-based benchmark.

\vspace{-2mm}
\paragraph{Limitations.} The factorization of transition probabilities described in \wasyparagraph\ref{sec:algorithm} introduces some bias towards the MI estimation. However, generally such bias is considered negligible, see \cite[\wasyparagraph4.5]{campbell2022continuous}, and reduces to zero with $N \rightarrow \infty$. Moreover, our training and estimation are simulation-free, so we can choose $N$ arbitrarily large.

\vspace{-2mm}
\paragraph{Future Work.} 
% This exploration into the Mutual Information estimation between finite state space random variables opens new perspective in bridge matching based estimators information theory.
This exploration of mutual information estimation for discrete state space variables opens new directions in information theory, particularly for \textbf{mixed state spaces}—combinations of discrete and continuous variables. Notably, bridge matching has already shown promise in mixed-type data generation, such as tabular data~\cite{guzmanexponential_tabbyflow}.

% This exploration into the Mutual Information estimation between discrete state space random variables opens a new perspective in information theory, including MI estimation, on \textbf{mixed state spaces}, i.e., the combination of discrete state space random variables and continuous state space random variables. Moreover, the bridge matching methodology already showed itself promising in mixed state space data generation, e.g., tabular data generation \cite{guzmanexponential_tabbyflow}.

% \paragraph{Reproducibility.} All the technical details that are required to reproduce the work are stated either in the main experimental section \wasyparagraph\ref{sec:experiments} or Appendix~\ref{appx:exp}. All the proofs for the provided theoretical results are provided either in Appendix~\ref{appx:proofs}.

\bibliography{references}
\bibliographystyle{iclr2025_delta}

\newpage
\appendix
% \section{Appendix}

\section{Proofs}\label{appx:proofs}

\subsection{Conditional Reciprocal Process as Markov Chain}

\begin{proof}
    Consider conditional reciprocal process $r_\pi$:
    \begin{gather}
        r_\pi(x_{\rm in}, x_1|x_0) = r_\pi(x_{t_1:t_{N + 1}}|x_0)
    \end{gather}
    Let us consider part of the reciprocal process $r(x_{t_1:t_{k}}|x_0)$, marginalize it by future timesteps $t_{k+1}:t_{N+1}$ and apply the reciprocal process definition \eqref{eq:reciprocal_def}:
    \begin{gather}
    r(x_{t_1:t_{k}}|x_0)= \sum_{x_{t_{k+1}}:x_{t_{N+1}}} r(x_{t_1:t_{N+1}}|x_0) \nonumber =  \sum_{x_{t_{k+1}}:x_{t_{N+1}}} q^{\rm ref}(x_{t_1:t_{N}}|x_0) \pi(x_1|x_0) = \nonumber \\ \sum_{x_{t_{k+1}}:x_{t_{N+1}}} q^{\rm ref}(x_{t_1:t_{N}}|x_0)\frac{q^{\rm ref}(x_1|x_0)}{q^{\rm ref}(x_1|x_0)} \pi(x_1|x_0)  = \sum_{x_{t_{k+1}}:x_{t_{N+1}}} q^{\rm ref}(x_{t_1:t_{N+1}}|x_0) \frac{\pi(x_1|x_0)}{q^{\rm ref}(x_1|x_0)}\label{eq:proof_reciprocal_1}
    \end{gather}

    Then recap that reference process $q^{\rm ref}$ is Markov:

    \begin{gather}
        q^{\rm ref}(x_{t_1:t_{N+1}}|x_0) = \left[\prod_{i=1}^kq^{\rm ref}(x_{t_i}|x_{t_{i-1}}, x_0)\right] \cdot  q^{\rm ref}(x_{t_{k+1}:t_{N+1}}|x_k, x_0) \label{eq:proof_Markov_q}
    \end{gather}

    Apply Markov chain form \eqref{eq:proof_Markov_q} for reciprocal process \eqref{eq:proof_Markov_q}:

    \begin{gather}
        r(x_{t_1:t_{k}}|x_0)= \left[\prod_{i=1}^kq^{\rm ref}(x_{t_i}|x_{t_{i-1}}, x_0)\right] \cdot \sum_{x_{t_{k+1}}:x_{t_{N+1}}} \frac{\pi(x_1|x_0)}{q^{\rm ref}(x_1|x_0)}q^{\rm ref}(x_{t_{k+1}:t_{N+1}}|x_k, x_0) \label{eq:proof_reciprocal_h}
    \end{gather}
    Then it is easy to see that last line term depends only on $x_t$:

    \begin{gather}
    h(x_k) = \sum_{x_{t_{k+1}}:x_{t_{N+1}}} \frac{\pi(x_1|x_0)}{q^{\rm ref}(x_1|x_0)}q^{\rm ref}(x_{t_{k+1}:t_{N+1}}|x_k, x_0)
    \end{gather}

    Finally let us take a look at the ratio between the parts of reciprocal process $r_\pi$ and apply \eqref{eq:proof_reciprocal_h}:

    \begin{gather}
        r_\pi(x_{t_k}|x_{t_1:t_{k-1}}, x_0) = \frac{r_\pi(x_{t_1:t_{k}}|x_0)}{r_\pi(x_{t_1:t_{k-1}}|x_0)} =  \frac{\left[\prod_{i=1}^kq^{\rm ref}(x_{t_i}|x_{t_{i-1}}, x_0)\right]}{\left[\prod_{i=1}^{k-1}q^{\rm ref}(x_{t_i}|x_{t_{i-1}}, x_0)\right]} \frac{h(x_k)}{h(x_{k-1})} = \nonumber \\ q^{\rm ref}(x_{t_k}|x_{t_{k-1}}, x_0) \frac{h(x_k)}{h(x_{k-1})} = r_\pi(x_{t_k}|x_{t_{k-1}}, x_0)
    \end{gather}

    Since $r_\pi(x_{t_k}|x_{t_1:t_{k-1}}, x_0) = r_\pi(x_{t_k}|x_{t_{k-1}}, x_0)$ it is evident that $r_\pi(x_{\rm in}, x_1|x_0)$ is a Markov chain.
\end{proof}

\subsection{Mutual Information Decomposition proof.} \label{appx:proof_main_theorem}

\begin{proof} Using the chain rule for KL divergence between stochastic processes $r_\pi^{\rm joint}$ and $r_\pi^{\rm ind}$:
\begin{gather}
    \KL{r_\pi^{\rm joint}(\cdot)}{r_\pi^{\rm ind}(\cdot)} = \KL{\pi(x_0)}{\pi(x_0)} +  \mathbb{E}_{\pi(x_0)}\left[\KL{r^{\rm joint}_{\pi|x_0}(\cdot)}{r_{\pi|x_0}^{\rm ind}(\cdot)} \right] \label{eq:proof_chain_rule_1}
\end{gather}
% \alex{x0 from what distribution?}
Note that  $r_\pi^{\rm joint}$ and $r_\pi^{\rm ind}$ share the same $\pi(x_0)$ marginal by construction, see \eqref{eq:recirpocal_ind} and \eqref{eq:recirpocal_joint}, so the first KL term is zero. Applying the chain rule for KL divergence between stochastic processes once more, and decomposing at the initial and terminal times of the reciprocal process, we obtain:
% \alex{DEEPSEEK GRAMMAR FIX}
\begin{multline}
    \KL{r_\pi^{\rm joint}(\cdot)}{r_\pi^{\rm ind}(\cdot)} = \KL{\pi(x_0, x_1)}{\pi(x_0)\pi(x_1)} + \\ \mathbb{E}_{\pi(x_0, x_1)}[\KL{r_{\pi}^{\rm joint}(\cdot|x_1, x_0)}{r_{\pi}^{\rm ind}(\cdot|x_1, x_0)}] \label{eq:proof_chain_rule_2}
\end{multline}
Recap that both $r_{\pi}^{\rm joint}(\cdot|x_1, x_0)$ and $r_{\pi}^{\rm ind}(\cdot|x_1, x_0)$ are $q^{\rm ref}(\cdot|x_0, x_1)$, see \eqref{eq:reciprocal_def}. Therefore, $\KL{r_{\pi}^{\rm joint}(\cdot|x_1, x_0)}{r_{\pi}^{\rm ind}(\cdot|x_1, x_0)} = 0$. Then by combining the \eqref{eq:proof_chain_rule_1} and \eqref{eq:proof_chain_rule_2} the following holds:
\begin{multline}
    \KL{r_\pi^{\rm joint}(\cdot)}{r_{\pi}^{\rm ind}(\cdot)} = \KL{\pi(x_0, x_1)}{\pi(x_0)\pi(x_1)} = \\\mathbb{E}_{\pi(x_0)}\left[ \KL{r_{\pi|x_0}^{\rm joint}(\cdot)}{r_{\pi|x_0}^{\rm ind}(\cdot)} \right] = \MI(X_0;X_1)
\end{multline}

Finally by presenting the conditional reciprocal processes $r_{\pi|x_0}^{\rm joint}$ and $r_{\pi|x_0}^{\rm ind}$ as the Markov chains:
\begin{multline}
    I(X_0;X_1) = \mathbb{E}_{\pi(x_0)}\left[
    \KL{r_{\pi|x_0}^{\rm joint}(\cdot)}{r_{\pi|x_0}^{\rm ind}(\cdot)} \right] =  \\ \mathbb{E}_{\pi(x_0)}\Big[ \sum_{n=1}^N \mathbb{E}_{x_{t_n}}\KL{r_\pi^{\rm joint}(x_{t_{n+1}}|x_{t_n}, x_0)}{r_\pi^{\rm ind}(x_{t_{n+1}}|x_{t_n}, x_0)} \Big]\nonumber
\end{multline}

\end{proof}

\section{Experimental Details}\label{appx:exp}

\begin{table*}[!t]
    \centering
    \begin{tabular}{lcccc}
        \toprule
        MI$_{true}$  & \ourestname{} (\textbf{ours}) & $f$-DIME-KL & $f$-DIME-H & $f$-DIME-G \\
        \midrule
        $1$ & \textbf{0.845} & 0.729 & 0.804 & 0.784 \\
        $2$ & 1.932 & 1.852 & \textbf{1.999} & 2.104 \\
        $3$ & \textbf{2.928} & 2.439 & 2.796 & 2.74 \\
        $4$ & 4.124 & 3.646 & 3.829 & \textbf{3.902} \\
        $5$ & \textbf{5.103} & 4.284 & 4.881 & 4.834  \\
        $6$ & 6.41 & 5.442 & 5.655 & \textbf{6.01}  \\
        \bottomrule
    \end{tabular}
    \caption{Results for the high-dimensional image benchmark with $16\times16$ images across all estimators. MI$_{\text{true}}$ denotes the ground-truth mutual information. All the other columns correspond to MI estimation results for a particular method. The closer to MI$_{\text{true}}$ the better. Best estimates are highlighted in \textbf{bold}.}
    \label{tab:exp_rectangle_16}
\end{table*}

\subsection{Posterior sampling}\label{appx:method_details}
% \sergei{add about posterior sampling}

%\paragraph{Parameterization via Posterior Sampling}

In line with the standard practice in diffusion models (see, e.g.~\cite{ho2020denoising}), we parameterize the transition probabilities \( r_\theta(x_{t_n} | x_{t_{n-1}}, x_0) \) using a posterior sampling scheme:
\[
r_\theta(x_{t_n} \mid x_{t_{n-1}},x_0) = \mathbb{E}_{\tilde{r}_\theta(\tilde{x}_1 \mid x_{t_{n-1}},x_0)}
\bigl[ q^{\mathrm{ref}}(x_{t_n} \mid x_{t_{n-1}}, \tilde{x}_1) \bigr],
\]
where \( \tilde{r}_\theta(\tilde{x}_1 \mid x_{t_{n-1}},x_0) \) is a learnable distribution. This parameterization assumes that sampling \( x_{t_n} \) given \( x_{t_{n-1}} \) and $x_0$ proceeds in two stages:
\begin{enumerate}
    \item Obtain probabilities of an endpoint \( \tilde{r}_\theta(\tilde{x}_1 \mid x_{t_{n-1}},x_0) \);
    \item Compute expectation of the reference process \( \mathbb{E}_{\tilde{x}_1}[q^{\mathrm{ref}}(x_{t_n} \mid x_{t_{n-1}}, \tilde{x}_1)] \) distribution and sample the next state \(x_{t_n} \).
\end{enumerate}

\subsection{General Other Methods Details.}

%In this part, we provide additional experimental details regarding other methods featured in our experiments. %~\cref{figure:compare_methods_images}.
%We report the NN architectures used for reference neural estimators in \wasyparagraph\ref{sec:image_bench}.

\paragraph{MINE, NWJ, InfoNCE.}
We use a nearly identical experimental framework to assess each approach within this category.
To approximate the critic function in experiments with synthetic data, we adopt the NN architecture from~\cite{butakov2024normflows}, which we also report in Table~\ref{table:synthetic_architecture} (``Critic NN'').

\begin{table*}[ht!]
    \center
    \caption{The NN architectures for variational methods (MINE, NWJ, InfoNCE, $f$-DIME) used to conduct the tests in \wasyparagraph\ref{sec:image_bench}. %\sergei{Write which method specifically to not confuse with out method.}
    }
    \label{table:synthetic_architecture}
    \begin{tabular}{cc}
    \toprule
    NN & Architecture \\
    \midrule
    \makecell{Critic NN, \\ low-dimensional data} &
        \small
        \begin{tabular}{rl}
            %$ \times 1 $: & Embedding(32) \\
            %$ \times 1 $: & Dense($2 \cdot \text{dim} \cdot 32$, 256), LeakyReLU(0.01) \\
            $ \times 1 $: & Dense($2 \cdot \text{dim}$, 256), LeakyReLU(0.01) \\
            $ \times 1 $: & Dense(256, 256), LeakyReLU(0.01) \\
            $ \times 1 $: & Dense(128, 1) \\
        \end{tabular} \\
    \\
    \makecell{Critic NN, \\ $ 16 \times 16 $ ($ 32 \times 32 $) \\ images} &
        \small
        \begin{tabular}{rl}
            $ \times 1 $: & [Conv2d(1, 16, ks=3), MaxPool2d(2), LeakyReLU(0.01)]$^{ \times 2 \; \text{in parallel}} $ \\
            $ \times 1(2) $: & [Conv2d(16, 16, ks=3), MaxPool2d(2), LeakyReLU(0.01)]$^{ \times 2 \; \text{in parallel}} $ \\
            $ \times 1 $: & Dense(256, 128), LeakyReLU(0.01) \\
            $ \times 1 $: & Dense(128, 128), LeakyReLU(0.01) \\
            $ \times 1 $: & Dense(128, 1) \\
        \end{tabular} \\
    \bottomrule
    \end{tabular}
\end{table*}

\paragraph{$f$-DIME} For all the $f$-DIME-H, $f$-DIME-KL, $f$-DIME-G we take the official implementation:

\begin{center}
    \url{https://github.com/nicolaNovello/fDIME}
\end{center}

\paragraph{Info-SEDD.}

We attempted to reproduce the Info-SEDD method \cite{foresti2025info} following the publicly available description, but were unable to obtain competitive performance under our experimental settings.

\subsection{Low Dimensional Benchmark Details}\label{appx:synthetic}
% \sergei{Add concrete details about the generation of row stoch matrix, exp kernel and e.t.c.}

\paragraph{Conditional Probability Matrix Generation}
The conditional probability matrix $
\Pi(x_1^d|x_0^d)$ for each dimension is generated as follows:

\begin{enumerate}
\item Create exponential kernel matrix using a Gaussian similarity measure:
\[
K_{ij} := \exp\left(-\frac{(i - j)^2}{2\sigma^2}\right), \quad i,j \in \{0,\dots,S-1\}
\]
where $\sigma$ controls the bandwidth of the kernel (we use $\sigma = 0.5$).
\item Multiply $K$ element-wise by a random matrix: 
\[
P_0 := K \odot \mathrm{Uniform}(0,1) + \epsilon,
\]
with $\epsilon = 10^{-12}$ for numerical stability.

\item Apply normalization to obtain a stochastic matrix: $\Pi(x_1^d|x_0^d) := Normalize(P_0)$.
\end{enumerate}

\paragraph{\ourestname{}}

We employ a DiT-based architecture with $\approx 300k$ parameters and following specifications:

Architecture: Embedding layer $\to$ $4$ Transformer blocks with rotary positional embeddings $\to$ linear projection. Each block has $4$ attention heads, hidden dimension $32$, conditioning dimension $32$. The approximate number of parameters in a neural network is $\approx 312k$.

Hyperparameters: Batch size $512$, learning rate $3\times10^{-4}$, $\alpha=10^{-4}$, $M = 10$ (number of inner samples in Algorithm~\ref{alg:MI_estimator}), training for $150$ epochs via KL loss.  Nvidia A100 was used to train the models.
In any setup, each run (one seed) took no longer than one GPU-hour to be completed.

% \sergei{Write the approximate number of parameters in a neural network.}

\paragraph{MINE, InfoNCE.}
%We use a nearly identical experimental framework to assess each approach within this category.
%To approximate $ T $ in experiments with synthetic images, we adopt the critic NN architecture from~\cite{butakov2024normflows}, which we also report in Table~\ref{table:synthetic_architecture} (``Critic NN''). 
%
The networks were trained via Adam optimizer with a learning rate $ 10^{-3} $, a batch size $ 512 $. 
For averaging, we used $ 5 $ different seeds.
Nvidia A100 was used to train the models.
%In any setup, each experiment ($ 6 \; \text{plot points} \times 4 \; \text{plots} \times 5 \; \text{seeds} $) took no longer than six GPU-days to be completed.
In any setup, each run (one seed) took no longer than one GPU-hour to be completed.

\paragraph{$f$-DIME}The neural network is MLP with Embedding layer to handle discrete data with $\approx 600k$ parameters. We follow the default parameters provided by the authors as much as possible: learning rate $2 \cdot 10^{-4}$, batch size $128$, number of gradient steps $50$k. Nvidia A100 was used to train the models.
In any setup, each run (one seed) took no longer than one GPU-hour to be completed.

\subsection{Image Based Benchmark}\label{appx:image_bench}

\paragraph{\ourestname{}}

We employ a UNet architecture with $\approx 6.8M$ parameters from \texttt{diffusers} library:

\begin{center}
    \url{https://huggingface.co/docs/diffusers/v0.21.0/en/api/models/unet2d}
\end{center}

Architecture: 2-channel input $\to$ Downsampling block (2 ResNet layers, 128 channels) $\to$ Downsampling block with spatial self-attention (2 ResNet layers, 128 channels) $\to$ Upsampling block with spatial self-attention (2 ResNet layers, 128 channels) $\to$ Upsampling block (2 ResNet layers, 128 channels) $\to$ 2-channel output The conditioning is done via concatenation of $x_{t_{n-1}}$ and $x_1$ by channel dimension.

Hyperparameters: batch size $128$, learning rate $3\times10^{-4}$, $\alpha=10^{-2}$, and $M=10$ inner samples in Algorithm~\ref{alg:MI_estimator}. We train for $30$ epochs using the KL objective, with an additional cross-entropy term weighted by $10^{-3}$ (see \cite[\wasyparagraph3.4]{austin2021structured}). 

Nvidia A100 was used to train the models.
In any setup, each run (one seed) took no longer than one GPU-hour to be completed.
\begin{figure}[!t]
    \centering

    \begin{subfigure}[t]{0.24\linewidth}
        \centering
        \includegraphics[width=\linewidth]{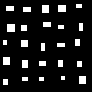}
        \caption{$x_0\sim\pi(x_0)$ }
    \end{subfigure}
    \begin{subfigure}[t]{0.24\linewidth}
        \centering
        \includegraphics[width=\linewidth]{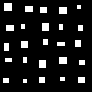}
        \caption{$x_1\sim\pi(x_1)$}
    \end{subfigure}
    \begin{subfigure}[t]{0.24\linewidth}
        \centering
        \includegraphics[width=\linewidth]{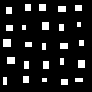}
        \caption{$x_1\sim  r_{\theta}(\cdot, 0)$}
    \end{subfigure}
    \begin{subfigure}[t]{0.24\linewidth}
        \centering
        \includegraphics[width=\linewidth]{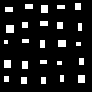}
        \caption{$x_1\sim  r_{\theta}(\cdot, 1)$}
    \end{subfigure}

    \caption{Qualitative samples, presented in $5\times2$ image grids, generated using image benchmark and \ourestname{}, i.e.,  $r_{\theta}(\cdot)$, at resolution $16\times16$.}
    \label{fig:rectangles_16}
\end{figure}

\paragraph{MINE, NWJ.}
The networks were trained via Adam optimizer
with a learning rate $ 10^{-3} $, a batch size $ 512 $. 
For averaging, we used $ 5 $ different seeds.
Nvidia A100 was used to train the models.
In any setup, each run (one seed) took no longer than two GPU-hours to be completed.
%The size of discriminator neural network is around $50$k parameters.

\begin{figure}[!t]
    \centering
    \includegraphics[width=0.4\linewidth]{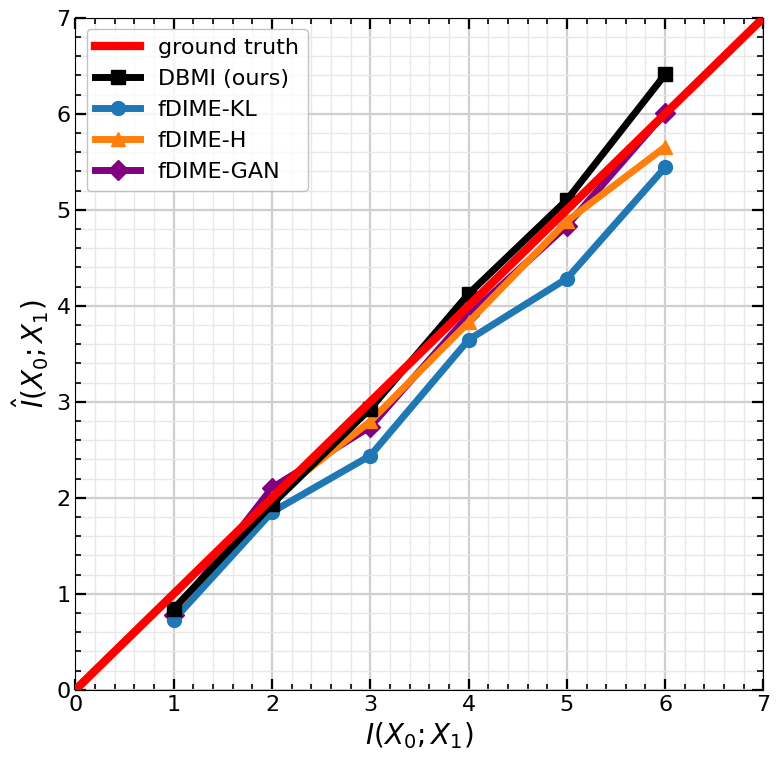}
    \caption{Comparison of estimated mutual information $\hat{\MI}(X_0; X_1)$ across methods against the ground-truth $\MI(X_0; X_1)$ (\textcolor{red}{red}) on a high-dimensional image benchmark with size $16\times16$.} 
    % Shaded regions denote 99\% confidence intervals computed across multiple random-seed runs.}
    \label{fig:exp_rectangle_16}
\end{figure}

\paragraph{$f$-DIME}
The neural network architectures for $32\times32$ and $16\times16$ image resolution setup are presented in Table~\ref{table:synthetic_architecture}, . We follow the default parameters provided by the authors as much as possible: learning rate $2 \cdot 10^{-4}$, batch size $128$, number of gradient steps $50$k. Nvidia A100 was used to train the models.
In any setup, each run (one seed) took no longer than one GPU-hour to be completed.

\section{Additional Experiments}\label{appx:add_exp}

\paragraph{Image Benchmark $16\times16$.} We test our \ourestname{} and $f$-DIME on image benchmark \wasyparagraph\ref{sec:image_bench} but with images being $16\times16$. In that case latent variables $\{X^i_0, X^i_1\}$ are constructed in the same way, but alphabet of discrete random variable is smaller, i.e., $V=5$. As well as minimum side length for rectangle $V^{\rm min}=5$. We vary ground truth mutual information: $MI_{\rm true} = \{1, 2, 3, 4, 5, 6\}$ and one can see the experiment results in Table~\ref{tab:exp_rectangle_16}, Figure~\ref{fig:exp_rectangle_16} and Figure~\ref{fig:rectangles_16}.

It is evident that both our \ourestname{} and $f$-DIME do solve the problem.

\end{document}